\documentclass[11pt]{article}
\usepackage{acl-hlt2011}
\usepackage{times}
\usepackage{amssymb}
\usepackage{latexsym}
\usepackage[pointedenum]{paralist}
\usepackage[noend]{algpseudocode}
\usepackage{amsmath}
\usepackage{amsfonts}
\usepackage{multirow}
\usepackage{bm}
\usepackage{url}
\usepackage{color}
\usepackage[pdftex]{graphicx}
\usepackage{tikz}
\usepackage{sparklines}
\usepackage{amsthm}
\newtheorem{definition}{Definition}
\newtheorem{theorem}{Theorem}
\newcommand{\der}[0]{\textbf{y}}

\newcommand{\ignore}[1]{}

\usepackage{amssymb}

\title{Minimum Error Rate Training and the Convex Hull Semiring\thanks{${}$ While preparing these notes, I discovered the work by Sokolov and Yvon (2011), who elucidated the semiring properties of the \textsc{mert} line search computation. Because they did not discuss the polynomial bounds on the growth of the values while running the inside algorithm, I have posted this as an unpublished manuscript.}}

\author{Chris Dyer\\
  School of Computer Science \\
  Carnegie Mellon University \\
  5000 Forbes Ave. \\
  Pittsburgh, PA 15213, USA \\
  {\tt cdyer@cs.cmu.edu}}
\date{}

\begin{document}
\maketitle
\begin{abstract}
We describe the line search used in the minimum error rate training algorithm \cite{mert} as the ``inside score'' of a weighted proof forest under a semiring defined in terms of well-understood operations from computational geometry. This conception leads to a straightforward complexity analysis of the dynamic programming \textsc{mert} algorithms of \newcite{macherey:2008} and \newcite{kumar:2009} and practical approaches to implementation.
\end{abstract}

\section{Introduction}

Och's (2003) algorithm for minimum error rate training (\textsc{mert}) is widely used in the direct loss minimization of linear translation models.  It is based on an efficient and optimal line search and can optimize non-differentiable, corpus-level loss functions. While the original algorithm used $n$-best hypothesis lists to learn from, more recent work has developed dynamic programming variants that leverage much larger sets of hypotheses encoded in finite-state lattices and context-free hypergraphs \cite{macherey:2008,kumar:2009,sokolov:2011}. Although \textsc{mert} has several attractive properties (\S\ref{sec:mert}) and is widely used in MT,  previous work has failed to explicate its close relationship to more familiar inference algorithms, and, as a result, it is less well understood than many other optimization algorithms.

In this paper, we show that the both the original \cite{mert} and newer dynamic programming algorithms given by \newcite{macherey:2008} and \newcite{kumar:2009} can be understood as  weighted logical deductions \cite{goodman:1999,lopez:2009,dyna} using weights from a previously undescribed semiring, which we call the \textbf{convex hull semiring} (\S\ref{sec:semirings}). Our description of the algorithm in terms of semiring computations has both theoretical and practical benefits: we are able to provide a straightforward complexity analysis and an improved DP algorithm with better asymptotic and observed run-time (\S\ref{sec:complexity}). More practically still, since many tools for structured prediction over discrete sequences support generic semiring-weighted inference \cite{openfst,joshua,cdec,dyna}, our analysis makes it is possible to add dynamic programming \textsc{mert} to them with little effort.

\section{Minimum error rate training}
\label{sec:mert} The goal of \textsc{mert} is to find a weight vector $\textbf{w}^* \in \mathbb{R}^d$ that minimizes a corpus-level loss $\mathcal{L}$ (with respect to a development set $\mathcal{D}$) incurred by a decoder that selects the most highly-weighted output of a linear structured prediction model parameterized by feature vector function $\textbf{H}$:
\begin{align*}
\textbf{w}^* &= \arg \min_{\textbf{w}}\mathcal{L}(\{\hat{\textbf{y}}_i^{\textbf{w}}\}, \mathcal{D}) \\
\{\hat{\textbf{y}}_i^{\textbf{w}}\} & = \arg \max_{\textbf{y}\in \mathcal{Y}(x_i)} \textbf{w}^{\top} \textbf{H}(x_i,\textbf{y})\quad \forall (x_i,\textbf{y}_i^{\textrm{gold}}) \in \mathcal{D} \\
\end{align*}
We assume that the loss $\mathcal{L}$ is computed using a vector error count function $\delta(\hat{y},y) \rightarrow \mathbb{R}^m$ and a loss scalarizer $L : \mathbb{R}^m \rightarrow \mathbb{R}$, and that the error count decomposes linearly across examples:\footnote{Nearly every evaluation metric used in NLP and MT fulfills these criteria, including F-measure, \textsc{bleu}, \textsc{meteor}, \textsc{ter}, \textsc{aer}, and \textsc{wer}. Unlike many dynamic programming optimization algorithms, the error count function  $\delta$ is not required to decompose with the structure of the model.} 
\begin{align*}
\mathcal{L}(\{\hat{\textbf{y}}_i^{\textbf{w}}\},\mathcal{D}) &= L\left(\sum_{i=1}^{|\mathcal{D}|}\delta(\hat{\textbf{y}}_i,\textbf{y}_i^{\textrm{gold}})\right) \\
\end{align*}
At each iteration of the optimization algorithm, \textsc{mert} choses a starting weight vector $\textbf{w}_0$ and a search direction vector $\textbf{v}$ (both $\in \mathbb{R}^d$) and determines which candidate in a set has the highest model score for \emph{all} weight vectors $\textbf{w}' = \eta\textbf{v} + \textbf{w}_0$, as $\eta$ sweeps from $-\infty$ to $+\infty$.\footnote{Several strategies have been proposed for selecting \textbf{v} and $\textbf{w}_0$. For an overview, refer to \newcite{galley:2011} and references therein.}

To understand why this is potentially tractable, consider any (finite) set of outputs $\{ \textbf{y}_j \} \subseteq \mathcal{Y}(x)$ for an input $x$ (e.g., an $n$-best list, a list of $n$ random samples, or the complete proof forest of a weighted deduction). Each output $\textbf{y}_j$ has a corresponding feature vector $\textbf{H}(x,\textbf{y}_j)$, which means that the \emph{model score} for each hypothesis, together with $\eta$, form a line in $\mathbb{R}^2$:
\begin{align*}
s(\eta)&=(\eta \textbf{v} + \textbf{w}_0)^{\top}\textbf{H}(x,\textbf{y}_j)\\
&= \eta \underbrace{\textbf{v}^{\top}\textbf{H}(x,\textbf{y}_j)}_{\textrm{slope}} + \underbrace{\textbf{w}_0^{\top}\textbf{H}(x,\textbf{y}_j)}_{y\textrm{-intercept}} \ \ \ . 
\end{align*}
The upper part of Figure~\ref{fig:setofderiv} illustrates how the model scores ($y$-axis) of each output in an example hypothesis set vary with $\eta$ ($x$-axis). The lower part shows how this induces a piecewise constant error surface (i.e., $\delta(\hat{\textbf{y}}^{\eta \textbf{v} + \textbf{w}_0},\textbf{y}^{\textrm{gold}})$). Note that  $\textbf{y}_3$ has a model score that is always strictly less than the score of some other output at all values of $\eta$. Detecting such ``obscured'' lines is useful because it is unnecessary to compute their error counts. There is simply no setting of $\eta$ that will yield weights for which $\textbf{y}_3$ will be ranked highest by the decoder.\footnote{Since $\delta$ need only be evaluated for the (often small) subset of candidates that can obtain the highest model score at some $\eta$, it is possible to use relatively computationally expensive loss functions. \newcite{zaidan:2009} exploit this and find that it is even feasible to solicit human judgments while evaluating $\delta$!}
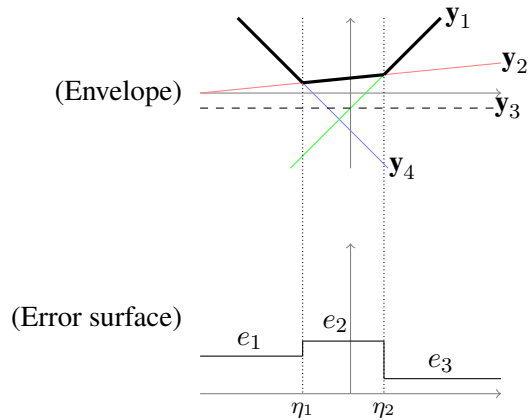
\begin{figure}[h]
\begin{center}
\begin{tikzpicture}
	\tikzstyle{axes} = [->,gray,thin]
	\tikzstyle{boundary} = [thin,densely dotted]
        
	\node [anchor=east,left=3pt] at (0,1) {(Envelope)};
	\draw[axes] (2,0) -- +(0,2); 
	\draw[axes] (0,1) -- +(4,0); 
	\draw (0.5,2) node (a 1){} -- (2.5,0) node (b 1){$\quad \der_4$}[blue!50!white];
	\draw (0,1.0) node (a 2){} -- (4,1.4) node (b 2){$\ \ \ \ \der_2$}[red!50!white];
	\draw (1.2,0) node (a 3){} -- (3.2,2) node (b 3){$\quad \ \der_1$}[green!80!white];
			\draw (0,1.2-0.4) node (a 4){} -- (4,1.2-0.4) node (b 4){$\ \ \der_3$}[dashed,black];


	\path (intersection of a 1--b 1 and a 2--b 2) coordinate (bound 1);
		\draw (0.5,2){} -- (bound 1){}[very thick,black];
	\draw[boundary] (bound 1 |- 0,2) -- (bound 1 |- 0,-3);

	\path (intersection of a 2--b 2 and a 3--b 3) coordinate (bound 2);
		\draw (bound 1){} -- (bound 2){}[very thick,black];
		\draw (bound 2){} -- (3.2,2){}[very thick,black];
	\draw[boundary] (bound 2 |- 0,2) -- (bound 2 |- 0,-3);
	

	\node [anchor=east,left=3pt] at (0,-2) {(Error surface)};
	\draw[axes] (2,-3) -- +(0,2); 
	\draw[axes] (0,-3) -- +(4,0); 
	\draw (0,-2.5) -- (bound 1 |- 0,-2.5); 
	\draw  (bound 1 |- 0,-2.5) --  (bound 1 |- 0,-2.3); 
	\draw (bound 1 |- 0,-2.3) -- (bound 2 |- 0,-2.3); 
	\draw (bound 2 |- 0,-2.3) -- (bound 2 |- 0,-2.8); 
	\draw (bound 2 |- 0,-2.8) -- (4,-2.8); 
                \node [anchor=south] at (bound 1 |- 0,-3.45) {\footnotesize{$\eta_1$}};
                \node [anchor=south] at (bound 2 |- 0,-3.45) {\footnotesize{$\eta_2$}};
                
          \node [anchor=north] at (0.66,-2.05) {{$e_1$}};
 \node [anchor=north] at (1.8,-1.82) {{$e_2$}};
\node [anchor=north] at (3.2,-2.35) {{$e_3$}};

\end{tikzpicture}
\end{center}
\vspace{-.5cm}
\caption{The model scores of a set of four output hypotheses $\{\der_1,\der_2,\der_3,\der_4\}$ under a linear model with parameters $\textbf{w} = \eta \textbf{v} + \textbf{w}_0$, inducing segments $(-\infty,\eta_1],[\eta_1,\eta_2],[\eta_2,\infty)$, which correspond (below) to error counts $e_1, e_2, e_3$.}
\label{fig:setofderiv}
\end{figure}

By summing the error surfaces for each sentence in the development set, a \emph{corpus-level} error surface is created. Then, by traversing this from left to right and selecting best scoring segment (transforming each segment's corpus level error count to a loss with $L$), the optimal  $\eta$ for updating $\textbf{w}_0$ can be determined.\footnote{\newcite{macherey:2008} recommend selecting the midpoint of the segment with the best loss, but \newcite{cer:2008} suggest other strategies.}

\subsection{Point-line duality}
The set of line segments corresponding to the maximum model score at every $\eta$ form an \emph{upper envelope}. To determine which lines (and corresponding hypotheses) these are, we turn to standard algorithms from computational geometry.  While algorithms for directly computing the upper envelop of a set of lines do exist, we proceed by noting that computing the upper envelope has as a dual problem that can be solved instead: finding the lower \emph{convex hull} of a set of points \cite{fourms}. The dual representation of a line of the form $y=mx+b$ is the point $(m,-b)$. This,  for a given output, $\textbf{w}_0$, $\textbf{v}$, and feature vector $\textbf{H}$, the line showing how the model score of the output hypothesis varies with $\eta$ can simply be represented by the \emph{point} $({\textbf{v}^{\top}\textbf{H},-\textbf{w}_0^{\top}\textbf{H}})$.

Figure~\ref{fig:duality} illustrates the line-point duality and the relationship between the primal upper envelope and dual lower convex hull. Usefully, the $\eta$ coordinates (along the $x$-axis in the primal form) where upper-envelope lines intersect and the error count changes are simply the \emph{slopes} of the lines connecting the corresponding points in the dual.

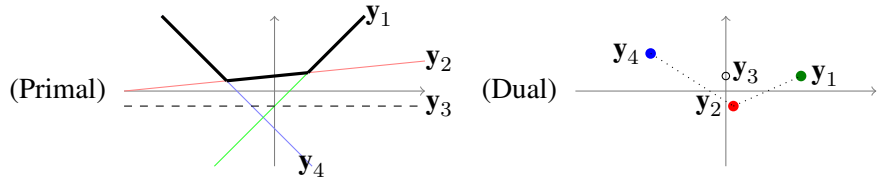
\begin{figure*}
\begin{center}
\begin{tikzpicture}
	\tikzstyle{axes} = [->,gray,thin]
	\tikzstyle{boundary} = [thin,densely dotted]

	\node [anchor=east,left=3pt] at (0,1) {(Primal)};
	\draw[axes] (2,0) -- +(0,2); 
	\draw[axes] (0,1) -- +(4,0); 
	\node [anchor=east,left=1pt] at (2.5,2) {$_{}$}; 
	\node [anchor=east,left=1pt] at (4.2,0.8) {$$}; 
	\draw (0.5,2) node (a 1){} -- (2.5,0) node (b 1){$\textbf{y}_4$}[blue!50!white];
	\draw (0,1.0) node (a 2){} -- (4,1.4) node (b 2){$\quad \textbf{y}_2$}[red!50!white];
	\draw (1.2,0) node (a 3){} -- (3.2,2) node (b 3){$\quad \textbf{y}_1$}[green!80!white];
		\draw (0,1.2-0.4) node (a 4){} -- (4,1.2-0.4) node (b 4){$\quad\textbf{y}_3$}[dashed,black];


	\path (intersection of a 1--b 1 and a 2--b 2) coordinate (bound 1);
		\draw (0.5,2){} -- (bound 1){}[very thick,black];

	\path (intersection of a 2--b 2 and a 3--b 3) coordinate (bound 2);
		\draw (bound 1){} -- (bound 2){}[very thick,black];
		\draw (bound 2){} -- (3.2,2){}[very thick,black];
	
	
	\node [anchor=east,left=3pt] at (6,1) {(Dual)};
	\draw[axes] (8,0) -- +(0,2);
	\draw[axes] (6,1) -- +(4,0.0);
	\node [anchor=east,left=1pt] at (8.5,2) {$$}; 
	\node [anchor=east,left=1pt] at (10.2,0.8) {$$}; 
	\fill[blue]  (8+-1,0.5+1)circle (2pt);
	\node [anchor=east,left=3pt] at (7.1,1.5) {$\textbf{y}_4$};
	\fill[red]  (8+0.1,-0.2+1)circle (2pt);
	\node [anchor=east,left=3pt] at (8.2,0.8) {$\textbf{y}_2$};
	\draw (8+-1,0.5+1){} -- (8+0.1,-0.2+1){}[dotted,black];
	\fill[green!50!black]  (8+1.0,0.2+1)circle (2pt);
	\node [anchor=west,right=3pt] at (8.9,1.2) {$\textbf{y}_1$};
	\draw (8+0.1,-0.2+1){} -- (8+1.0,0.2+1){}[dotted,black];

	\draw[black]  (8,0.2+1)circle (1.4pt);

		\node [anchor=west,right=3pt] at (7.85,1.25) {$\textbf{y}_3$};

\end{tikzpicture}
\end{center}
\vspace{-.5cm}
\caption{Primal and dual forms of a set of lines. The upper envelope is shown with heavy line segments in the primal form. In the dual, primal lines are represented as points, with upper envelope lines corresponding to points on the lower convex hull. The dashed line $\textbf{y}_3$ is obscured from above by the upper envelope in the primal and (equivalently) lies above the lower convex hull of the dual point set.}
\label{fig:duality}
\end{figure*}

\section{The Convex Hull Semiring}
\label{sec:semirings}
\begin{definition}
A \textbf{semiring} $K$ is a quintuple $\langle \mathbb{K}, \oplus, \otimes, \overline{0}, \overline{1} \rangle$  consisting of a set $\mathbb{K}$, an addition operator $\oplus$ that is associative and commutative, a multiplication operator $\otimes$ that is associative, and the values $\overline{0}$ and $\overline{1}$ in $\mathbb{K}$, which are the additive and multiplicative identities, respectively.  $\otimes$ must distribute over $\oplus$ from the left or right (or both), i.e., $a \otimes (b \oplus c) = (a \otimes b) \oplus (a \otimes c)$ or $(b \oplus c) \otimes a = (b \otimes a) \oplus (c \otimes a)$.  Additionally, $\overline{0} \otimes u = \overline{0}$ must hold for any $u\in\mathbb{K}$.  If a semiring $K$ has a commutative $\otimes$ operator, the semiring is said to be \emph{commutative}.  If $K$ has an idempotent $\oplus$ operator (i.e., $a \oplus a = a$ for all  $a \in \mathbb{K}$), then K is said to be \emph{idempotent}.
\end{definition}

\begin{definition} \textbf{The Convex Hull Semiring}. Let $(\mathbb{K},\oplus,\otimes,\overline{0},\overline{1})$ be defined as follows:
\begin{tabular}{c|l}
\hline
$\mathbb{K}$ & A set of points in the plane that are\\
& $\quad$the extreme points of a convex hull.\\
$A \oplus B$ & $\textrm{\emph{conv}}\left[A \cup B \right]$ \\
$A \otimes B$ & convex hull of the Minkowski sum, i.e., \\
&$\quad \textrm{\emph{conv}} \{  (a_1 + b_1, a_2 + b_2) \mid$ \\
& $\quad\quad(a_1,a_2) \in A \ \wedge\  (b_1,b_2) \in B \}$\\
$\overline{0}$ & $\emptyset$ \\
$\overline{1}$ & $\{ ( 0,0 ) \}$ \\
\end{tabular}
\end{definition}

\begin{theorem}
The Convex Hull Semiring fulfills the semiring axioms and is commutative and idempotent.
\end{theorem}
\begin{proof}
To show that this is a semiring, we need only to demonstrate that commutativity and associativity hold for both addition and multiplication, from which distributivity follows.  Commutativity ($A \cdot B = B \cdot A$) follows straightforwardly from the definitions of addition and multiplication, as do the identities.  Proving associativity is a bit more subtle on account of the conv operator. For multiplication, we rely on results of \newcite{krein:1940}, who show that
$$\textrm{conv }[A +_{\textrm{Mink.}} B] = \textrm{conv }[\textrm{conv } A +_{\textrm{Mink.}} \textrm{conv } B] \ \ .$$
For addition, we make an informal argument that a context hull circumscribes a set of points, and convexification removes the interior ones. Thus, addition continually expands the circumscribed sets, regardless of what their interiors were, so order does not matter. Finally, addition is idempotent since $\textrm{conv }[A \cup A] = A$.
\end{proof}

\section{Complexity}
\label{sec:complexity}
Shared structures such as finite-state automata and context-free grammars encode an exponential number of different derivations in polynomial space. Since the values of the convex hull semiring are themselves sets, it is important to understand how their sizes grow. Fortunately, we can state the following tight bounds, which guarantee that growth will be worst case linear in the size of the input grammar:
\begin{theorem} $|A \oplus B| \le |A| + |B|$.\end{theorem}
\begin{theorem} $|A \otimes B| \le |A| + |B|$.\end{theorem}
\noindent The latter fact is particularly surprising, since multiplication appears to have a bound of $|A|\times |B|$. The linear (rather than multiplicative) complexity bound for Minkowski addition is the result of Theorem~13.5 in \newcite{fourms}. From these inequalities, it follows straightforwardly that the number of points in a derivation forest's total convex hull is upper bounded by $|E|$.\footnote{This result is also proved for the lattice case by \newcite{macherey:2008}.}

\section*{Acknowledgements}
We thank David Mount for suggesting the point-line duality and pointing us to the relevant literature in computational geometry and Adam Lopez for the TikZ \textsc{mert} figures.

\bibliographystyle{acl}
\bibliography{biblio}
\end{document}